\documentclass[conference, a4paper]{IEEEtran}
\IEEEoverridecommandlockouts
\renewcommand{\baselinestretch}{0.97}
\usepackage{amsthm}
\newtheorem{theorem}{Theorem}
\newtheorem{remark}{Remark}

\newtheorem{lemma}{Lemma}
\usepackage{subfigure} 
\usepackage{algorithm}
\usepackage{url}
\usepackage{mathtools}
\usepackage{float}
\usepackage{dsfont} 
\usepackage{cite}
\usepackage{amsmath,amssymb,amsfonts}
\allowdisplaybreaks
\usepackage{algorithmic}
\usepackage{graphicx}
\usepackage{textcomp}
\usepackage{xcolor}
\theoremstyle{definition}

\newcommand{\remove}[1]{}
\usepackage{mdframed} 
\usepackage{etoolbox}
\usepackage[a4paper,top=1.9cm,bottom=4.4cm,left=1.35cm,right=1.35cm]{geometry}

\def\BibTeX{{\rm B\kern-.05em{\sc i\kern-.025em b}\kern-.08em
    T\kern-.1667em\lower.7ex\hbox{E}\kern-.125emX}}
\begin{document}

\title{Recommendation aided Caching using Combinatorial  Multi-armed Bandits
}

\remove{
\title{An UCB-based Algorithm for Caching with Recommendation
}
}
\author{\IEEEauthorblockN{ Pavamana K J}
\IEEEauthorblockA{\textit{Department of Electronics System Engineering} \\
\textit{Indian Institute of Science}\\
Bengaluru, India \\
pavamanak@iisc.ac.in }
\and
\IEEEauthorblockN{ Chandramani Kishore Singh}
\IEEEauthorblockA{\textit{Department of Electronics System Engineering} \\
\textit{Indian Institute of Science}\\
Bengaluru, India \\
chandra@iisc.ac.in }
}

\maketitle

\begin{abstract}
We study content caching with recommendations in a wireless network where the users are connected through a base station equipped with a finite-capacity cache.  We assume a fixed set of contents with  unknown user preferences and content popularities. The base station can cache a subset of the contents and can also recommend subsets of the contents to different users in order to encourage them to request the recommended contents. Recommendations, depending on their acceptability, can thus be used to increase  cache hits. We first assume  that the users' recommendation acceptabilities are known and 
formulate the  cache hit optimization problem as a combinatorial multi-armed bandit (CMAB). We 
propose a 
UCB-based algorithm to decide which contents to cache and recommend and provide an upper bound on the regret of this algorithm. Subsequently, we consider a more general scenario where the users' recommendation acceptabilities are also unknown and propose another  
UCB-based algorithm that learns these as well. We numerically demonstrate the performance of our algorithms and compare these to state-of-the-art algorithms.

\end{abstract}


\section{Introduction}

As various intelligent devices, e.g., smart vehicles, home appliances, and mobile devices continue to evolve alongside a plethora of innovative applications, \remove{such as real-time news updates, high-definition video streams, and software updates,} mobile wireless communications are witnessing an unprecedented surge in traffic. This surge is characterized by a significant amount of redundant and repeated information, which limits the capacity of both the fronthaul and backhaul links. \remove {Based on the findings in Ericsson's report \cite{ericssonmobilityreport}, it is projected that global mobile data traffic will escalate to 325EB per month by 2028, marking a nearly fourfold increase compared to 2022} Utilizing content caching at the network edge, including base stations, presents a promising strategy for addressing the rapidly growing demand for data traffic and enhancing user satisfaction.

Recently, recommendation mechanisms have demonstrated their effectiveness in enhancing the caching performance of conventional networks. By reshaping users' content request behaviors, these mechanisms encourage edge caches to become more user-friendly and network-friendly \cite{chatzieleftheriou2017caching}. Notably, studies by \cite{zhou2010impact} and \cite{krishnappa2015cache} revealed that 50\% of YouTube video views are driven by recommendation systems, a figure that rises to 80\% for Netflix \cite{gomez2015netflix}. However, independently deploying recommendation and caching strategies may not fully leverage their potential. This is because recommendations influence users' content request behaviours, subsequently impacting caching decisions.


In our problem, we model caching with recommendations as a Combinatorial Multi-Armed Bandit (CMAB). In the CMAB framework, an agent selects a subset of arms in each round and receives the sum of the rewards associated with them. The goal is to maximize this cumulative reward over time. In our context, each  content corresponds to an arm in the CMAB, and we must decide which contents to cache and which to recommend to maximize the total rewards.

However, the users' preference distribution i.e., the distribution of content requests without recommendations is unknown and cannot be directly learned from observed requests, as recommendations influence users' content request behaviors, subsequently impacting caching decisions. Furthermore, caching decisions also influence recommendations. Therefore, it is necessary to jointly solve the recommendation and caching problems. We propose a UCB-based algorithm for both scenarios: when the users' recommendation acceptabilities are known and when they are unknown. 

\subsection{Related Works}
In recent times, numerous research studies have suggested leveraging content recommendations to enhance caching efficiency. In \cite{chatzieleftheriou2018jointly} and \cite{chatzieleftheriou2017caching}, recommendations are used to guide user requests towards cached content that aligns with their interests. Furthermore, recent developments have explored employing content recommendation to fulfill content requests by offering alternative or related content, as discussed in \cite{sermpezis2018soft}, \cite{song2018making}, and \cite{costantini2020approximation}. Considering the advantages of having recommendations in caching, there have been a lot of works on joint content caching and recommendation \cite{tsigkari2022approximation},  \cite{chatzieleftheriou2017caching} and \cite{chatzieleftheriou2018jointly}. In \cite{chatzieleftheriou2017caching}, it was initially demonstrated that user preferences could be reshaped to enhance caching performance. Subsequent studies have further confirmed these findings, showing notable improvements in performance compared to traditional caching methods that do not integrate recommendation systems.

Most of the above works assume that users’ preference
distributions are known apriori. Moreover, they all also assume that users’ recommendation acceptances are known.


\paragraph*{\textbf{Our Contributions}} We offer a new view point to the joint caching and recommendation problem.  Unlike most of the existing works we let the users' preference distributions be unknown and propose an online learning-based solution. In particular, we frame the cache hit optimization problem as a combinatorial multi-armed bandit (CMAB) and give an algorithm to learn these distributions. Since the users' preferences are influenced by recommendations, we develop an estimation framework that only observes cached contents' requests, making the problem more challenging than the classical full-observation settings. Ours is the first work to consider unknown  recommendation  acceptabilities of the users. We give a UCB-based algorithm to learn these effectively.

\paragraph*{Organisation of the paper} The rest of the paper is organised as follows. In Section \ref{System Model} we present the system model for content caching with recommendation problem and formulate it as CMAB. In Section \ref{Algorithm Design}, we describe the algorithm to solve the problem and also provide upper bound on the regret. In Section \ref{Sec:Unk-W}, we propose another algorithm when users' recommendation acceptabilities are unknown. In Section \ref{Numerical Results}, we show the cache hit performance of our algorithm and also compare it with few existing algorithm. Section \ref{Conclusion and Future work} concludes the section with some insight on future directions.

\section{System Model and Caching Problem} \label{System Model}
We consider a wireless network where the users are connected to a single base station~(BS) which in turn is connected to content servers via the core network.
The content providers have a set of $N$ contents, ${\cal N} = \{1,2,\cdots,N\}$ of equal size at the servers.  There are $U$ users; ${\cal U} = \{ 1,2, 3, ..., U\}$ denotes the set of users.
The BS has a {\it cache} where it can store up to $C$ contents can recommend up to $R$ contents to each user where $R \leqslant C$. Different sets of contents can be recommended to different users.The reason for choosing $R \leq C$ is that users may ignore the recommendation if we recommend too many items. \remove{ We  consider a recommendation policy to improve the {\it cache hit} performance.} 

We assume a slotted system.  The caching and recommending decisions are taken at the slot boundaries. We denote the caching vector by $\mathbf{Y} = \{ y_1, y_2, y_3,...., y_N  \}$ where $ y_i \in \{ 0,1 \}. $ $y_i = 0$ implies that the $i^{th}$ content is not cached and $y_i = 1$ implies that it is cached. Similarly, we denote the recommending vector for user $u $ by $\mathbf{X_u} = \{x_{u1} ,x_{u2},...,x_{uN}   \}$ where $x_{ui} \in \{ 0,1 \}$. If $x_{ui} = 1 $ then the $i^{th}$ content is recommended to user $u$, else it is not recommended to user $u$.

We aim to employ a recommendation strategy to enhance {\it cache hit} performance. A cache hit happens whenever the requested content is stored in the cache. \remove{\textcolor{red}{What is meant by performance?}} We assume that the base station (BS) does not know the users' preferences, and the BS controls the recommendation system. The relationship between the recommendations and the user preferences is modelled as follows. We denote $p_u^{\text {pref }}(i)$ to be the probability that user $u$ would request file $i$  in the absence of recommendations. We denote  $\mathcal{R}_u$ to be the set of files recommended to user $u$. We assume that the recommendations to user $u$ induce a preference distribution $p_u^{\text{rec}}(\cdot)$; $0 \leq p_u^{\text {rec }}(i) \leq 1, \forall i$, $p_u^{\text {rec }}(i)=0, \forall i \notin \mathcal{R}_u$, and $\sum_{i=1}^R p_u^{\text{rec}}(i)=1$. By jointly considering $p_u^{\text {pref }}(\cdot)$ and $p_u^{\text {rec }}(\cdot)$, the request probability of user $u$ for file $i$ is then modeled as
\begin{equation} \label{Request Distribution}
   p_u^{\text {req }}(i)=w_u^{\text {rec }} p_u^{\text {rec }}(i)+\left(1-w_u^{\text {rec }}\right) p_u^{\text {pref }}(i).
\end{equation}
Here $0 \leq w_u^{\mathrm{rec}} \leq 1$ is the probability that user $u$ accepts the recommendations from the recommendation system. From \eqref{Request Distribution}, we can see that the final request distribution of the user is a convex combination of its original preference distribution and the preference distribution induced by the recommendations. Thus, the recommendation system shapes the probability distribution for file-requesting. We can see from \eqref{Request Distribution} that the request probabilities for the recommended contents are boosted and the request probabilities for non-recommended contents are decreased.  There are other ways of combining $p_u^{\text {rec}}(i)$ and $p_u^{\text{pref}}(i) $, e.g., $p_u^{\text {req }}(i)= \max\{p_u^{\text{rec }}(i), p_u^{\text{pref }}(i)\}$~\cite{verhoeyen2012optimizing}, but a convex  combination is the most general way of combining and  has been used in most of the related works \cite{chatzieleftheriou2018jointly ,tsigkari2022approximation , xie2024joint,fu2022joint,fu2021caching}.\remove{\textcolor{red}{cite all the works that have used it}.  }


In this work, we assume that the users select contents from the recommended list uniformly, i.e.,
\begin{equation}
\label{Recommending Distribution}
p_u^{\text {rec }}(i) = \begin{cases}
1/ R  & \text{if } x_{ui} = 1, \\
0 &  \text{otherwise.} \\
\end{cases}
\end{equation}
\remove{\textcolor{red}{Can we relax this assumption? If yes, we must mention this.}} This uniform distribution over recommended content has also been considered in several other works \cite{tsigkari2022approximation,chatzieleftheriou2017caching}. We can also consider the case where positions of contents in the recommended list govern thier probabilities of getting selected by the user, e.g., the contents at the top have higher probabilities of getting selected than those at the lower positions~\cite{qi2018optimizing, zhou2010impact, krishnappa2015cache,zheng2020cooperative}\remove{(see \textcolor{red}{citations} for such models)}. In Remark \ref{Remark:Zipf}, we argue that our analysis also applies to this more general case when position in the recommended list impacts the request probability.

Furthermore, $w_u^{\mathrm{rec}}$ captures the percentage of time user $ u$ accepts the recommendation. It has been reported that $50\% $ and $80\%$  of the requests come from recommendations on YouTube and Netflix, respectively \cite{zhou2010impact,krishnappa2015cache,gomez2015netflix}. But, it can, of course, differ across the users. Also, $w_u^{\mathrm{rec}}$ can change over time depending on the user's interest. It is possible to estimate $w_u^{\mathrm{rec}}$ when these are unknown. We will also consider the case where $w_u^{\mathrm{rec}}$'s are unknown.

\subsection{Content Caching Problem}

In this section, we formulate the optimal caching and recommendation problem as a CMAB problem aiming at maximizing the cache hits. \remove{ \textcolor{red}{This should come earlier, where cache hit is introduced.}}The sum of the cache hit probabilities across all the users is given by
\begin{align}
& \sum_u \sum_i y_i p_u^{\text {req }}(i) \nonumber \\
& = \sum_u \sum_i y_i \bigg[w_u^{\text{rec}} \frac{x_{ui}}{R} + ( 1 - w_u^{\text{rec}}) p_u^{\text{pref}}(i)\bigg], \label{eqn:objective}
\end{align} 
where~\eqref{eqn:objective} follows from the request distributions~\eqref{Request Distribution} and Assumption~\eqref{Recommending Distribution}. 
So, our goal is to solve the following optimization problem to get the optimal caching and  recommendations
\begin{align}
    \max_{\mathbf{X}, \mathbf{Y}} \sum_u \sum_i y_i & \bigg[w_u^{\text{rec}} \frac{x_{ui}}{R} + ( 1 - w_u^{\text{rec}}) p_u^{\text{pref}}(i)\bigg] \label{Optimization Problem}\\
    \text{subject to} \quad  & \sum_i y_i \leq C  \label{Constraint 1}   \\
     &  \sum_i x_{u i} \leq R \quad \forall u \in U \label{Constraint 2} \\
     & x_{ui} \in \{ 0,1\} , y_i \in \{ 0,1\} \quad \forall u,i \label{Constraint 3}
\end{align}
 Constraint~\eqref{Constraint 1} accoutns for the cache capacity, and~\eqref{Constraint 2} for the maximum number of recommendations to a user. Constraint~\eqref{Constraint 3} indicates that contents cannot be partially cached or recommended\remove{~(see \textcolor{red}{citations} where the authors consider partial caching and recommendations).}

In Problem \eqref{Optimization Problem}, user preference distributions in the absence of recommendation, $p_u^{\text{pref}}(i)$ for all $u, i$, are unknown. One  way of solving this problem is to first estimate $p_u^{\text{pref}}(i)$ for all $u, i$, and then solve the optimization problem. But recommendations affect the numbers of requests for different contents, so we cannot directly estimate $p_u^{\text{pref}} (i)$. We can only estimate the request distributions $p_u^{\text{req}}(i)$. Also, we can only observe the requests of the cached contents and hence it is modeled as CMAB( in CMAB , only the rewards of the pulled arms are observed ). This is different from full observation setting \remove{\cite{bura2021learning}} where we can observe the requests of all the contents irrespective of whether those are cached or not. So in the next section, we introduce a UCB-based algorithm that estimates $p_u^{\text{req}}(i)$  for all $u, i$, and solves Problem~\eqref{Optimization Problem}.

\section{Algorithm Design} \label{Algorithm Design}
This section introduces an UCB-based algorithm, Algorithm~\ref{alg:ucb-rec}, to solve Problem~\eqref{Optimization Problem}.  
\remove{
Using Assumption~\eqref{Recommending Distribution}, we can simplify the objective function~\eqref{Optimization Problem} as follows

\begin{align}
&\max_{\mathbf{X}, \mathbf{Y}} \sum_u \sum_i y_i  \bigg[x_{u i}\bigg(w_u^{\text{rec}} p_u^{\text{rec}}(i)+(1-w_u^{\text{rec}}) p_u^{\text{pref}}(i)\bigg)+  \nonumber\\
    &  \quad \quad \quad \quad \quad \quad \quad \left(1-x_{u i}\right)\left(1-w_u^{\text{rec}}\right) p_u^{\text{pref}}(i)\bigg] \nonumber \\
 &= \max _{\mathbf{X}, \mathbf{Y}} \sum_u \sum_i y_i \bigg[x_{u i} w_u^{\text {rec}} p_u^{\text {rec }}(i)+\left(1-w_u^{\text {rec }}\right) p_u^{\text {pref }}(i)\bigg] \nonumber \\
    & = \max _{\mathbf{X}, \mathbf{Y}} \sum_u \sum_i \frac{y_i x_{u i} w_u^{\text {rec }}}{R}+\sum_i \sum_u y_i\left(1-w_u^{\text {rec }}\right) p_u^{\text {pref }}(i)  \label{Simpified Optimization Problem}
\end{align}
}
On  a careful look at  Problem~ \eqref{Optimization Problem} we see that, for given $y_i$s, the first term is maximized when $x_{ui} = 1$ only for those $i$ for which $y_i = 1$. Hence, we should recommend cached contents only. Moreover, under this choice of $x_{ui}$s, the first term is independent of the choice of $y_i$s. To maximize the second term, we choose $Y^*$ as follows
\begin{align} \label{eqn: Cache Solution}
    \quad \mathbf{Y}^*=\underset{\mathbf{Y}}{\operatorname{argmax}} \sum_i y_i \sum_u\left(1-w_u^{\text {rec }}\right) p_u^{\text {pref }}(i).
\end{align}
Unfortunately, $p_u^{\text {pref }}(i)$ are unknown and we cannot estimate these as we explained in the previous section. If we knew $p_u^{\text {pref }}(i)$, the optimal solution would be to choose top $C$ contents with the highest  $\sum_u\left(1-w_u^{\text {rec }}\right) p_u^{\text {pref }}(i)$. 
\remove{
, i.e.
\begin{align}
    \mathbf{Y}^* &= \underset{C}{\operatorname{argmax}} \left\{ \sum_u p_u^{\text{pref}}(1) , \sum_u p_u^{\text{pref}}(2) ,..., \sum_u p_u^{\text{pref}}(N) \right\} \label{Optimal Solution 1}
\end{align} 
}
Let us denote the optimal set of contents to be cached by $\mathcal{Y}=\left\{i: y_i^*=1\right\}$ and the set of contents cached at time $t$ according to  Algorithm~\ref{alg:ucb-rec} by $y(t)=\left\{i: y_i^t=1\right\}$. Now we define the regret upto time $T$ for not caching the optimal contents $\mathcal{Y}$ as follows
\begin{align}
\operatorname{R}(T) & = T \sum_u \sum_i y_i^* p_u^{\text {req }}(i) - \sum_{t=1}^T \sum_u \sum_{i} y_i^t p_u^{\text {req }}(i) \nonumber \\
&  = T \sum_{i \in \mathcal{Y}} \sum_u p_u^{\text {req }}(i)-\sum_{t=1}^T \sum_{i \in y(t)} \sum_{u} p_u^{\text {req }}(i)  \nonumber\\
&= T \sum_{i \in \mathcal{Y}} p_i-\sum_{t=1}^T \sum_{i \in y(t)} p_i \label{Regret:Definition}
\end{align}
where $p_i \coloneqq \sum_u p_u^{\text {req }}(i)$. If we knew $p_i$s for all $i$, the optimal solution $\mathcal{Y}$ could also be obtained as following: arrange the contents in the decreasing order of $p_i$s and choose the top $C$ contents.
\remove{
\begin{align}
    \mathcal{Y} &= \underset{C}{\operatorname{argmax}} \left\{ \sum_u p_u^{\text{req}}(1) , \sum_u p_u^{\text{req}}(2) ,..., \sum_u p_u^{\text{req}}(N) \right\} \label{ Optimal Solution 2}
\end{align}
}
In the absence of knowledge of $p_i$'s,  we  estimate  these for all $i$ at each time.   Let $Z_i(t)$  denote  the number of request for a cached content  $i$
     at time $t$. Recall that we can observe requests only for cached contents. So we set $$Z_i(t) = 0 \quad \forall i \notin y(t).$$ We also define
\begin{equation}
    p_i(t) =  \frac{{\sum_{s = 1}^{t}Z_i(s)}}{ n_i(t)}  \label{Eq: Estimation}.
\end{equation}
where $n_i(t)$  is the number of times content  $i$  is cached.UCB-based algorithm  construct the UCB index $U_i(t)=p_i(t)+D_i(t)$ for each content $i$,
where $D_i(t)$ is the confidence interval quantifying the uncertainty in $p_i(t)$ at time $t$. The key step in our analysis is ingenious choice of confidence interval (see \eqref{Confidence Interval}). We denote by $\bar{w}^{\text{rec}}$ the mean of $w_u^{\text{rec }}$'s over all the users.
The confidence interval is defined as follows 
\begin{align} \label{Confidence Interval}
    D_i(t)= U \sqrt{\frac{2\left( 1 - \bar{w}^{\text{rec}}\right)^{1/\eta} \log t}{n_i(t)}+\frac{\alpha \bar{w}^{\text{rec}}}{2 n_i(t)}} .
\end{align}
The UCB-based algorithm chooses the top $C$ contents with the largest indices and recommends any $R$ contents out of these.

\begin{algorithm}[H]
\caption{UCB with Recommendation}\label{alg:ucb-rec}
\begin{algorithmic}[1]
\REQUIRE $C$: Cache capacity, $R$: Maximum number of recommended contents,\\
Set of contents $\mathcal{N} = \{1,2,3,...,N\}$
\ENSURE Updated statistics $n_i(t)$ and $p_i(t)$
\STATE Initialization: $n_i(0)=1 ,p_i(0)=0  \quad \forall i \in \mathcal{N}$
\FOR{$t=1,2, \ldots, T$}
    \STATE Compute UCBs $$U_i(t)=p_i(t)+D_i(t)$$ for all $i \in \mathcal{N}$ using (\ref{Eq: Estimation}) and (\ref{Confidence Interval}).
    \STATE Pick Top $C$ contents with the highest UCBs to cache.\\
    Call this set $y(t)$
    \remove{
    $$C(t)=\underset{C}{\operatorname{argmax}}\left\{U_k(t): \forall k\right\}$$
    }
    \STATE For each user, pick any $R$ contents 
    from the cached \\ contents to recommend
    \STATE Observe the number of requests for cached contents and call \textbf{Algorithm \ref{alg:update-stats}}
\ENDFOR
\end{algorithmic}
\end{algorithm}

\begin{algorithm}
\caption{Subroutine:Update Statistics}\label{alg:update-stats}
\begin{algorithmic}[1]
\STATE \textbf{Input:} Cached contents $y(t)$, observed requests $Z(t)$
\STATE \textbf{Output:} Updated statistics $n_i(t)$ and $p_i(t)$
\FORALL{$i \in y(t)$}
    \STATE Update $n_i(t) \gets n_i(t-1) + 1$
    \STATE Update $p_i(t) \gets \frac{p_i(t-1)n_i(t-1) + Z_i(t)}{n_i(t)}$
\ENDFOR
\FORALL{$i \notin y(t)$}
    \STATE Set $n_i(t) \gets n_i(t-1)$ and $p_i(t) \gets p_i(t-1)$
\ENDFOR
\end{algorithmic}
\end{algorithm}

\begin{remark} \label{Remark: CI Explanation}
    We have chosen confidence interval as in \eqref{Confidence Interval} because we want these to be low when $\bar{w}^{\text{rec}} \approx 1$ and high when $\bar{w}^{\text{rec}} \approx 0$. This is because the users tend to accept the recommendations when $\bar{w}^{\text{rec}} \approx 1$. Since we know the distribution over recommended contents, which is the uniform distribution, there is less uncertainty about $p_i$'s when $\bar{w}^{\text{rec}} \approx 1$. Similarly, the users ignore the recommendation when $\bar{w}^{\text{rec}} \approx 0$. In this case, the users follow their preference distributions $p_u^{\text{pref}}(i)$ ignoring the recommendations. Then, Problem~\eqref{Optimization Problem} becomes standard CMAB without recommendation. Our algorithm is formally presented as Algorithm~\ref{alg:ucb-rec} below.
\end{remark}



\begin{remark} \label{Remark:Zipf}
  We have assumed the recommendation induced preference distributions $p_u^{rec}$'s to be uniform distributions~(see~\eqref{Recommending Distribution}). However, the proposed UCB-based algorithm and the regret bound  also apply to more general induced preference distributions wherein preference for a content depends on the number of recommended contents $R$ and the tagged content's position in the recommendation list. The key observation is that the optimal set of cached contents continues to be characterized by~\eqref{eqn: Cache Solution}. Several works have modeled the induced preference distributions as  Zipf distributions, given by~(see~\cite{qi2018optimizing, zheng2020cooperative}).  \begin{equation}\label{Eqn: Zipf Distr}
    p_u^{\text {rec }}(i) = \frac{ \sum_{k = 1}^{R} l_{uik }k^{-\beta_u} }{ \sum_{j = 1}^{R} j^{-\beta_u}}.
   \end{equation} 
   where $\beta_u$ is the Zipf distribution parameter associated with user $u$ and $l_{uik } \in \{0,1\}$ 
   are defined as follows.
   \begin{equation*}
   l_{uik}  = \begin{cases}
      1 & \text{if  content $i$ is at the $k^{\text{th}}$ position in user $u$'s} \\
      & \text{recommendation list,} \\
      0 & \text{otherwise.}
      \end{cases}
      \end{equation*}
      For $\beta_u = 0$ Zipf distributions
      reduce to the uniform distribution 
in~\eqref{Recommending Distribution}. In Section~\ref{Numerical Results}, we evaluate  the \it{cache hit} performance of the proposed algorithm for both uniform and Zipf distributions.
\end{remark}

\begin{remark}
    Our proposed algorithm suggests that we recommend only the cached contents. In general, recommendations to users should also depend on their preferences. In particular, they are expected not to be far from the users' 
    preferences. This can be captured through the following constraints on recommendations~(see~\cite{9424982}). 
    \begin{equation} \label{eqn: Recommendation Quality}
        \sum_{i \in \mathcal{N}} x_{ui} p_u^{pref}(i) \geq Q_u, \quad \forall u \in \mathcal{U},
    \end{equation}
   where $Q_u$'s are user-specific thresholds. The quantities on the left in~\eqref{eqn: Recommendation Quality} are referred to as recommendation qualities. Incorporating these constraints in recommendation decisions is a part of our future work. 
\end{remark}










The following Lemma \ref{Lemma 1} offers equivalent definition of regret $\operatorname{R}(T)$ and is used in the proof of Theorem \ref{Theorem: Regret Results}.Let us define, $\hat{p_i} \coloneqq \sum_u\left(1-w_u^{\text{rec}}\right) p_u^{\text {pref }}(i)$, $\Delta_{e,k} \coloneqq \hat{p_e} - \hat{p_k} \quad \forall e\in \mathcal{Y} , k \in \mathcal{N} \setminus \mathcal{Y}$.

\begin{lemma} \label{Lemma 1}The regret $\operatorname{R}(T)$ \eqref{Regret:Definition} can also be written as follows
   \begin{align*}
        \operatorname{R}(T) = T \sum_{i \in \mathcal{Y}} \hat{p_i}-\sum_{t=1}^T \sum_{i \in y(t)} \hat{p_i}
    \end{align*}
\end{lemma}
\begin{proof} 

Please refer to our extended version \cite{j2024recommenadationaidedcachingusing}

\end{proof}

 The following theorem characterizes the regret of Algorithm~\ref{alg:ucb-rec}.

\remove{\textcolor{red}{Define $\Delta_{e,k}$ before stating the theorem. Move Definition~3 to the proof of the theorem.}}

\begin{theorem} \label{Theorem: Regret Results}
   The expected regret of Algorithm~\ref{alg:ucb-rec} is bounded as
    \begin{align*}
        E[ \operatorname{R}(T)] & \leqslant U^2 \sum_{k \in \mathcal{N} \setminus \mathcal{Y}} \left[ \frac{16\left( 1 - \bar{w}^{\text{rec}}\right)^{1/\eta} \log T}{\Delta_{\min,k} } + \frac{4\alpha \bar{w}^{\text{rec}}}{\Delta_{\min,k}} \right] \\
        & \quad +  \sum_{k \in \mathcal{N} \setminus \mathcal{Y}} \sum_{e \in \mathcal{Y}} \Delta_{e,k} \left( \frac{2e^{-\alpha \bar{w}^{\text{rec}}}}{4\left(1 - \bar{w}^{\text{rec}}\right)^{1/\eta}-1} \right)
    \end{align*}
for $ \Bar{w}^{\text{rec}} \leqslant 1-\frac{1}{4^\eta}$, where $\Delta_{\min,k} \coloneqq \min_{e \in \mathcal{Y}} \Delta_{e,k}$, $\alpha >0$, and $\eta > 0$.
\end{theorem}

\begin{remark}
     If we use Algorithm \ref{alg:ucb-rec} without recommendations, i.e., setting $w_u^{\text{rec}} = 0$ for all $u$, then we get the standard CMAB. An upper bound on its regret can be derived following a similar approach as in~\cite[Theorem~2]{DBLP:journals/corr/KvetonWAEE14}. It turns out to be $$  \operatorname{R}(T) \leq \sum_{e \in \mathcal{N} \setminus \mathcal{Y}} \frac{16 U^2}{\Delta_{\min,e}} \log T+\sum_{k \in \mathcal{N} \setminus \mathcal{Y}} \sum_{e \in \mathcal{Y}} \Delta_{e, k} \frac{4}{3} \pi^2 .$$
     \remove{
     Proofs are excluded because similar steps follows.
     }
    Theorem~~\ref{Theorem: Regret Results} offers a lower regret bound.
    Expectedly, recommending a subset of popular contents while caching these contents helps in achieving lower regret.
\end{remark}

\begin{remark}
    The proof of Theorem \ref{Theorem: Regret Results} is motivated by~\cite[Theorem~2]{DBLP:journals/corr/KvetonWAEE14}. However, we use different confidence intervals, which warrant different arguments at several steps. 
\end{remark}

\begin{proof}
   Please refer to our extended version \cite{j2024recommenadationaidedcachingusing}
\end{proof}

\section{Unknown users’
recommendation acceptability $w_u^{\text{rec}}$}\label{Sec:Unk-W}

In this section, we delve into the intriguing scenario where the values of $w_u^{\text{rec}}$'s remain unknown. This uncertainty presents both a challenge and an opportunity for exploration. At each time step \( t \), we estimate the value of $\bar{w}^{\text{rec}}$, leveraging these estimates to compute the UCB indices for all contents.

In this framework, we make the simplifying assumption that the same set of contents is recommended to all users at each time step. This uniformity enables us to analyze the effectiveness of our recommendations in a controlled manner. To facilitate our discussion, we define the sets $\bar{\tau}_i^{t}$ and $\tau_i^{t}$ as follows:

\begin{align*}
    \tau_i^{ t} &= \left\{ s : s \leqslant t \text{ such that } i \text{ is cached and } \right. \\
    &  \left. \quad \quad \quad \quad \quad \text{recommended at time } s \right\} \\
    \bar{\tau}_i^{ t} &= \left\{ s : s \leqslant t \text{ such that } i \text{ is cached and } \right. \\
    & \left. \quad \quad \quad \quad \quad  \text{not recommended at time } s \right\}.
\end{align*}
We estimate $\bar{w}^{\text{rec}}$ as follows
\begin{equation} \label{Eqn:W_rec Est}
    \bar{w}^{\text{rec}}(t) = \frac{R}{U N}\sum_{i=1}^{N} \left( \Tilde{p}_i(t) - \hat{p}_i(t)\right)
\end{equation}
where 
\begin{align*}
    \hat{p}_i(t) &= \frac{\sum_{s:s \in \bar{\tau}_i^{t}} Z_i(s)}{|\bar{\tau}_i^{t}|},  \\
    \Tilde{p}_i(t) &= \frac{\sum_{s:s \in \tau_i^{t}} Z_i(s)}{|\tau_i^{t}|}. 
\end{align*}

Now we define the confidence interval as follows,
\begin{equation} 
\label{Eqn: New Confidence Interval}
 \hspace{-0.1in}  D_i^{\bar{w}^{\text{rec}}(t)}(t)= U \sqrt{\frac{2\left( 1 - \bar{w}^{\text{rec}}(t)\right)^{1/\eta} \log t}{n_i(t)}+\frac{\alpha \bar{w}^{\text{rec}}(t)}{2 n_i(t)}} . 
\end{equation}

Finally, we propose a new algorithm, Algorithm~\ref{alg:ucb-unk-w}, for optimal caching and recommendation.
Regret analysis of Algorithm~\ref{alg:ucb-unk-w} is the subject matter of our future work. In Section, \ref{Numerical Results} we show that the estimated $\bar{w}^{\text{rec}}(t)$~(see~\eqref{Eqn:W_rec Est}) converges to the actual mean $\bar{w}^{\text{rec}}$ pretty quickly. We also show the regret performance of Algorithm \ref{alg:ucb-unk-w} in Section \ref{Numerical Results}.


\begin{algorithm}[H]
\caption{UCB  with Recommendation and Unknwon $w_u^{rec}\text{'s}$ 
}\label{alg:ucb-unk-w}
\begin{algorithmic}[1]
\REQUIRE cache capacity $C$, maximum number of recommended contents $R$, set of contents $\mathcal{N}$
\ENSURE Updated statistics $n_i(t)$ and $p_i(t) \quad \forall i \in \mathcal{N}$ 
\STATE Initialization: $n_i(0)=1 ,p_i(0)=0  \quad \forall i \in \mathcal{N}$
\FOR{$t=1,2, \ldots, T$}
    \STATE Estimate  $\bar{w}^{\text{rec}}(t) $ using \eqref{Eqn:W_rec Est}
    \STATE Compute UCBs $$U_i(t)=p_i(t)+D_i^{\bar{w}^{\text{rec}}(t)}
    (t)$$ for all $i \in \mathcal{N}$ using (\ref{Eq: Estimation}) and \eqref{Eqn: New Confidence Interval}.
    \STATE Pick Top $C$ contents with the highest UCBs to cache.\\
    Call this set $y(t)$
    \remove{
    $$C(t)=\underset{C}{\operatorname{argmax}}\left\{U_k(t): \forall k\right\}$$
    }
    \STATE Pick any $R$ contents 
    from $y(t)$ to recommend
    \STATE Observe the number of requests for cached contents and call \textbf{Algorithm \ref{alg:update-stats}}
\ENDFOR
\end{algorithmic}
\end{algorithm}

\begin{figure*}[!htbp] 
    \centering


    \subfigure[Regret of Algorithm~\ref{alg:ucb-rec} vis-à-vis existing algorithms for uniform preference distribution over recommended contents.]
    {\includegraphics[width=2in,height=1.3in]{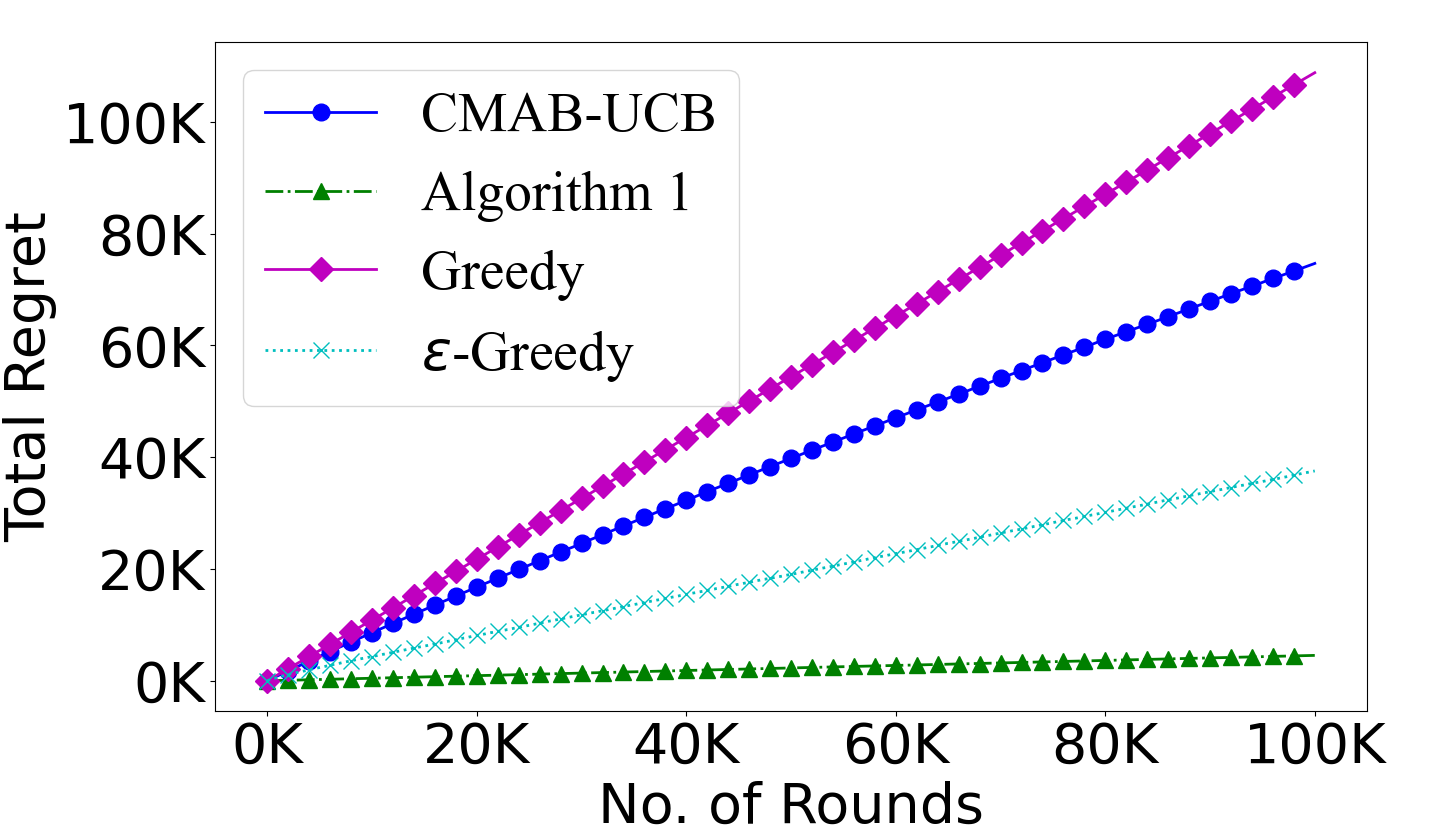}} \label{fig:Performance}
    \hfill
    \subfigure[Regret of Algorithm~\ref{alg:ucb-rec} for different values of $w_u^{\text{rec}}$.]
    {\includegraphics[width=2in,height=1.3in]{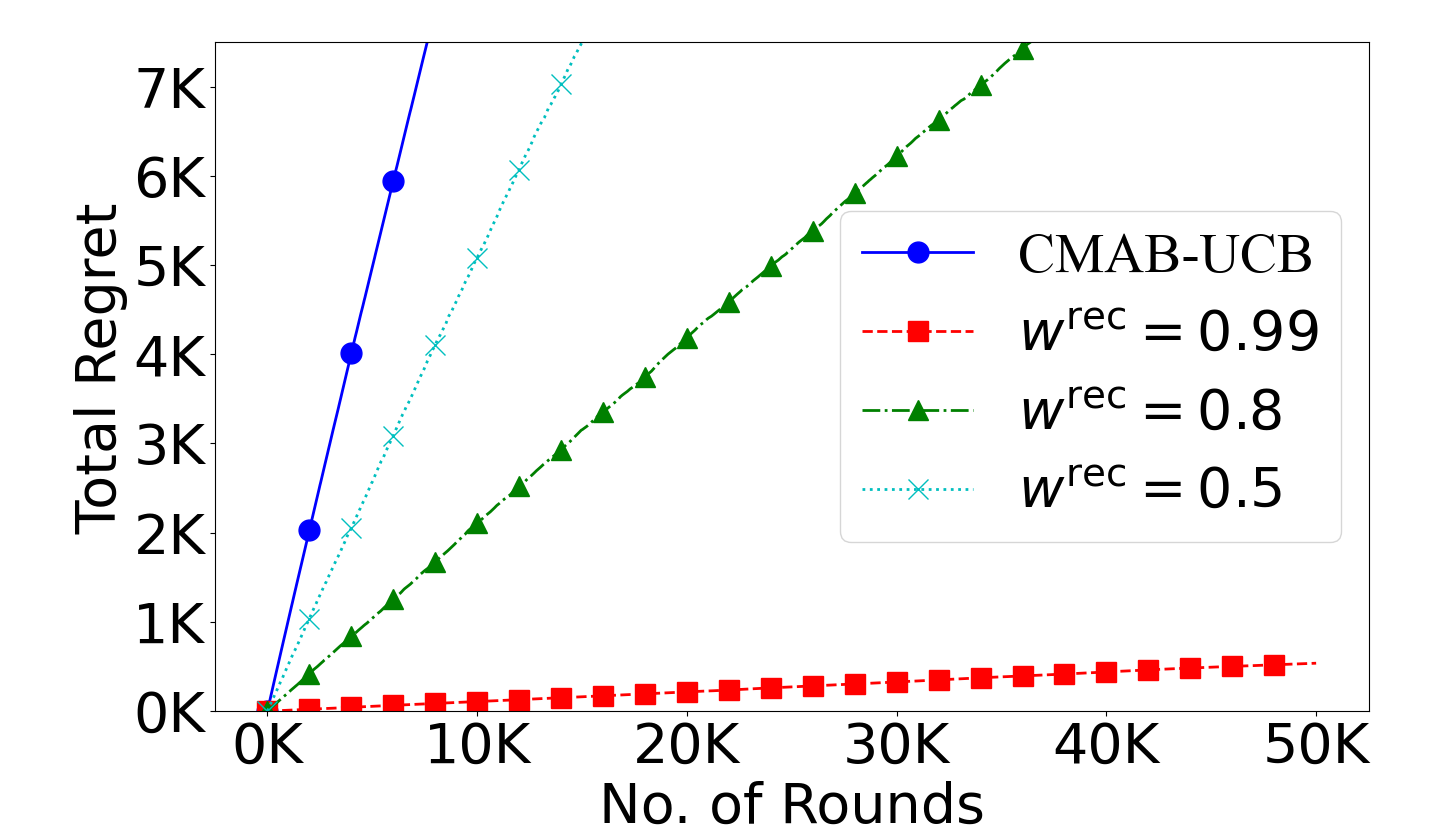} \label{fig:W_rec}}
    \hfill
    \subfigure[Regret of Algorithm~\ref{alg:ucb-rec} for different numbers of users.]
    {\includegraphics[width=2in,height=1.3in]{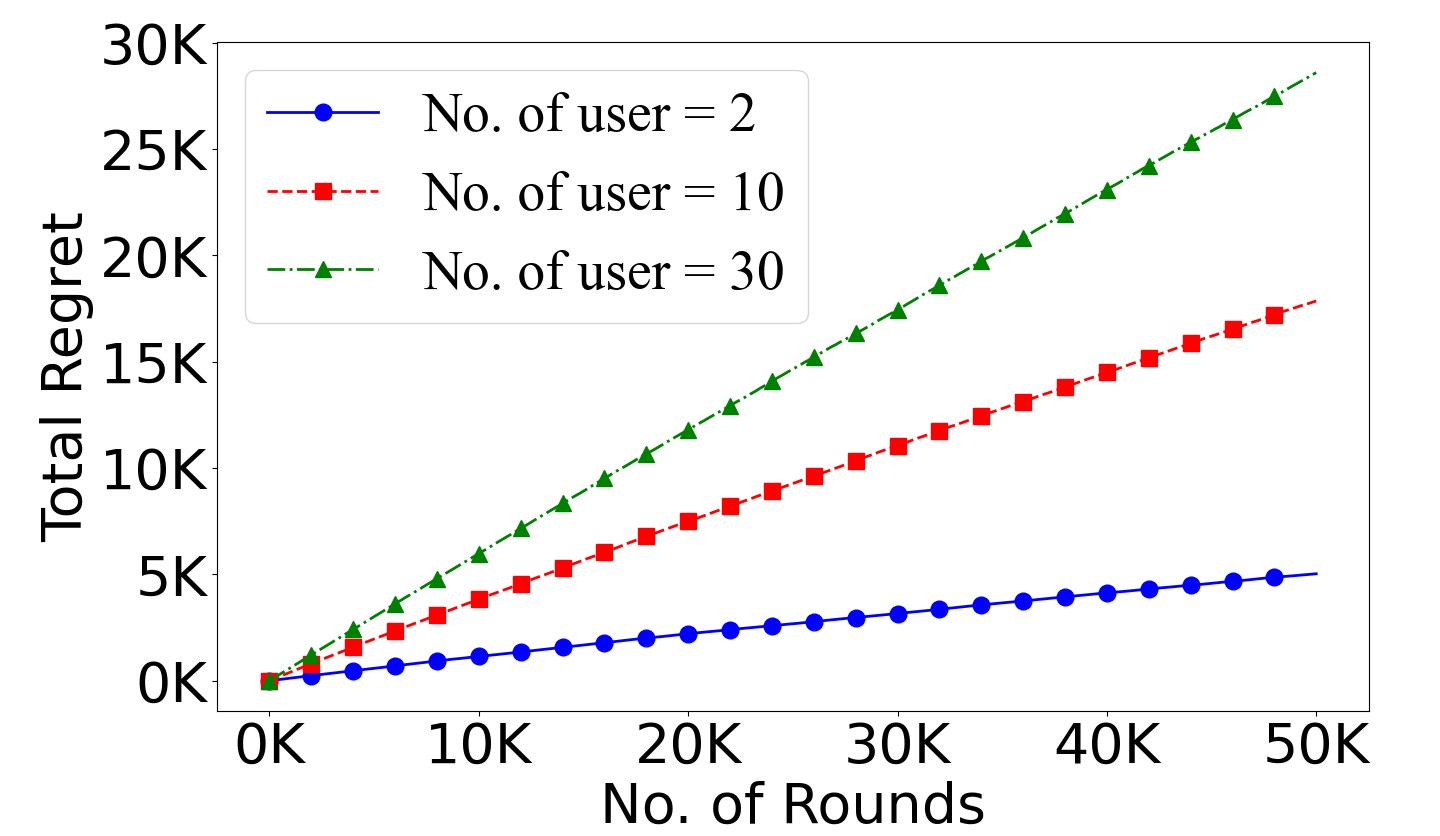} \label{fig:alpha}}
    \vspace{-0.2cm} 

    \subfigure[Regret of Algorithm~\ref{alg:ucb-rec} vis-à-vis existing algorithms for Zipf preference distribution and uniform distribution over recommended contents.]
    {\includegraphics[width=2in,height=1.3in]{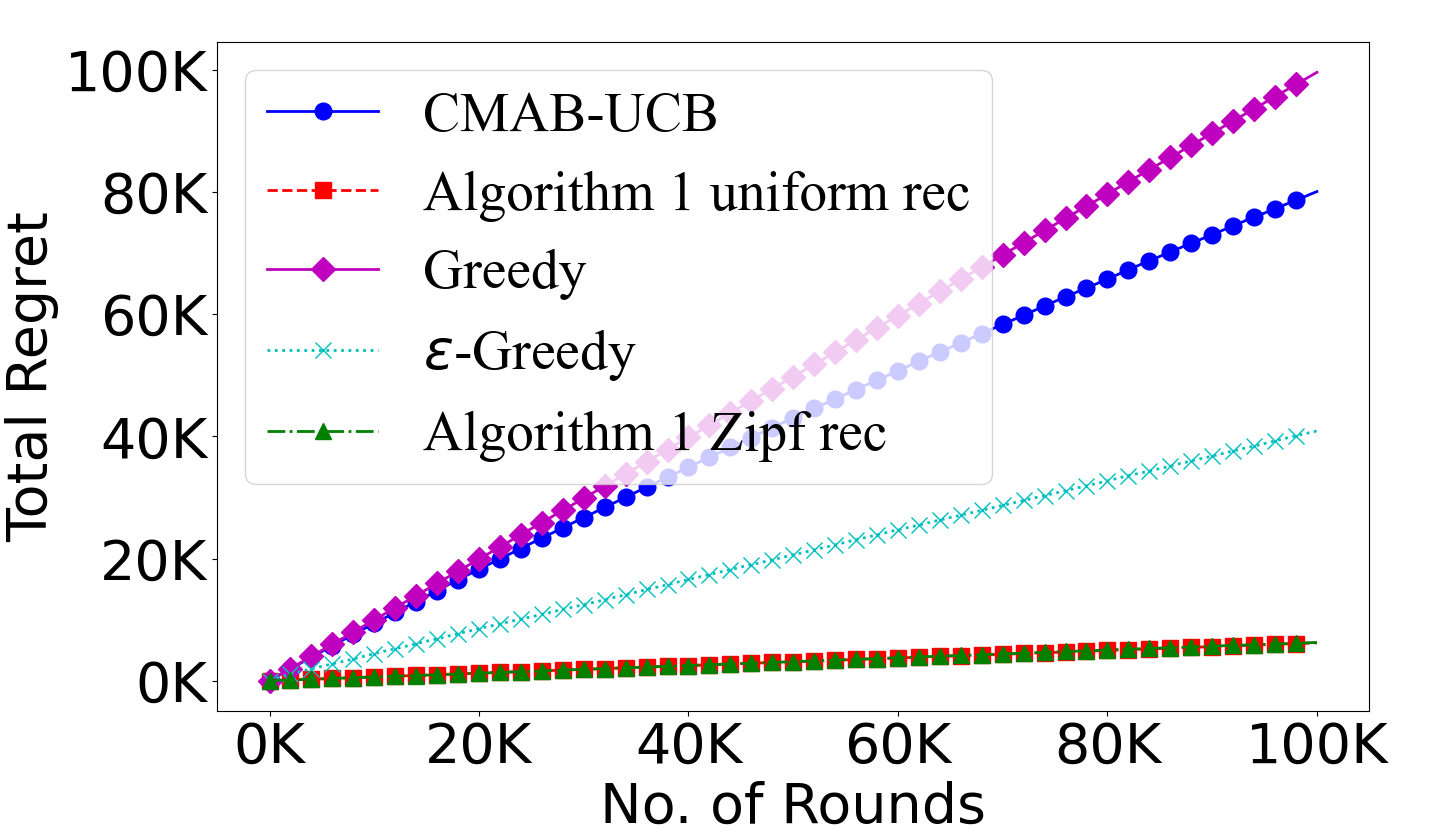} \label{fig:Performance_Zipf}}
    \hfill
    \subfigure[Convergence of $\bar{w}^{\text{rec}}$ under Algorithm~\ref{alg:ucb-unk-w}.]
    {\includegraphics[width=2in,height=1.3in]{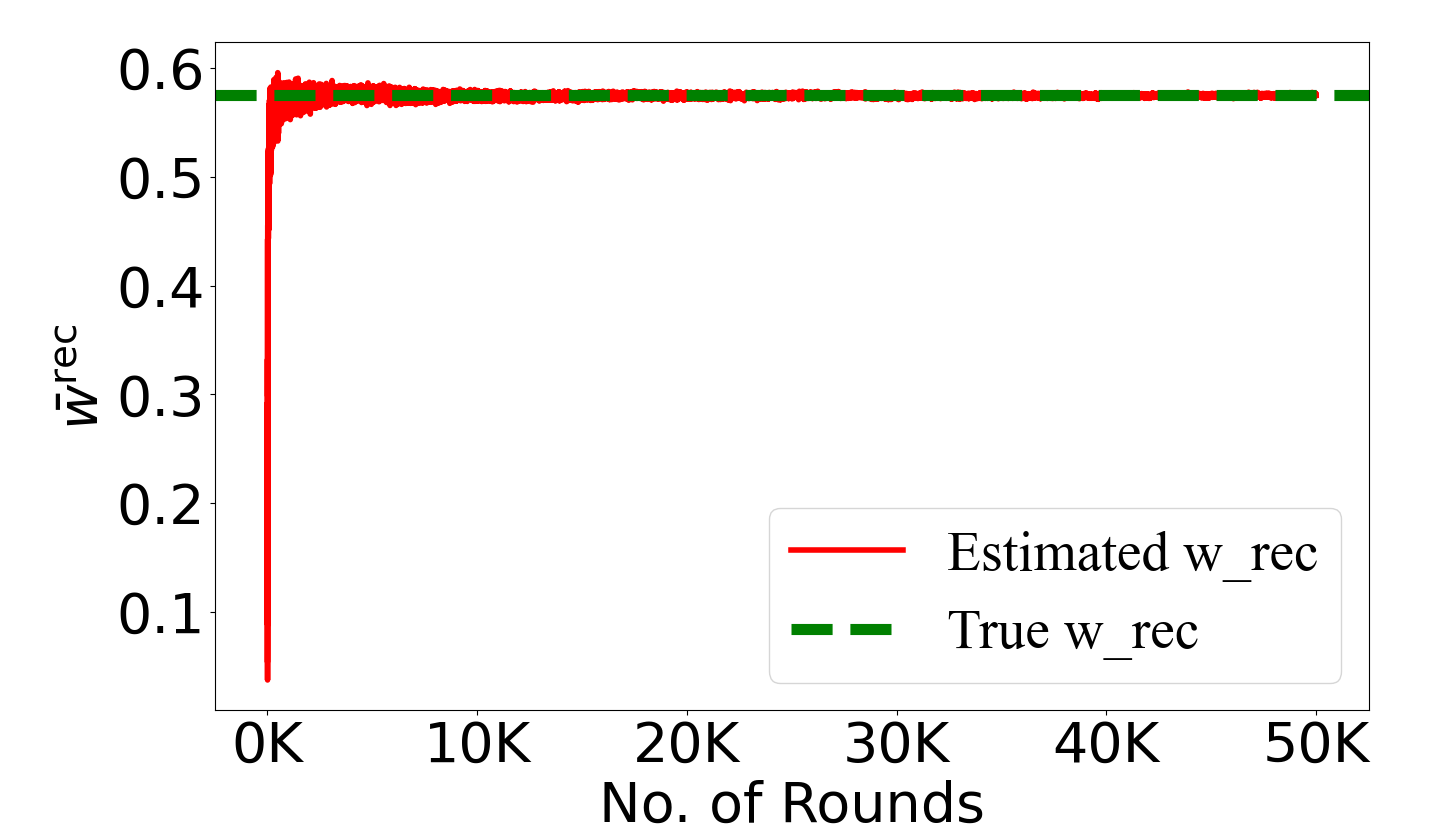} \label{fig:Convergence}}
    \hfill
    \subfigure[Regret comparison of Algorithms~\ref{alg:ucb-rec} and~\ref{alg:ucb-unk-w}.]
    {\includegraphics[width=2in,height=1.3in]{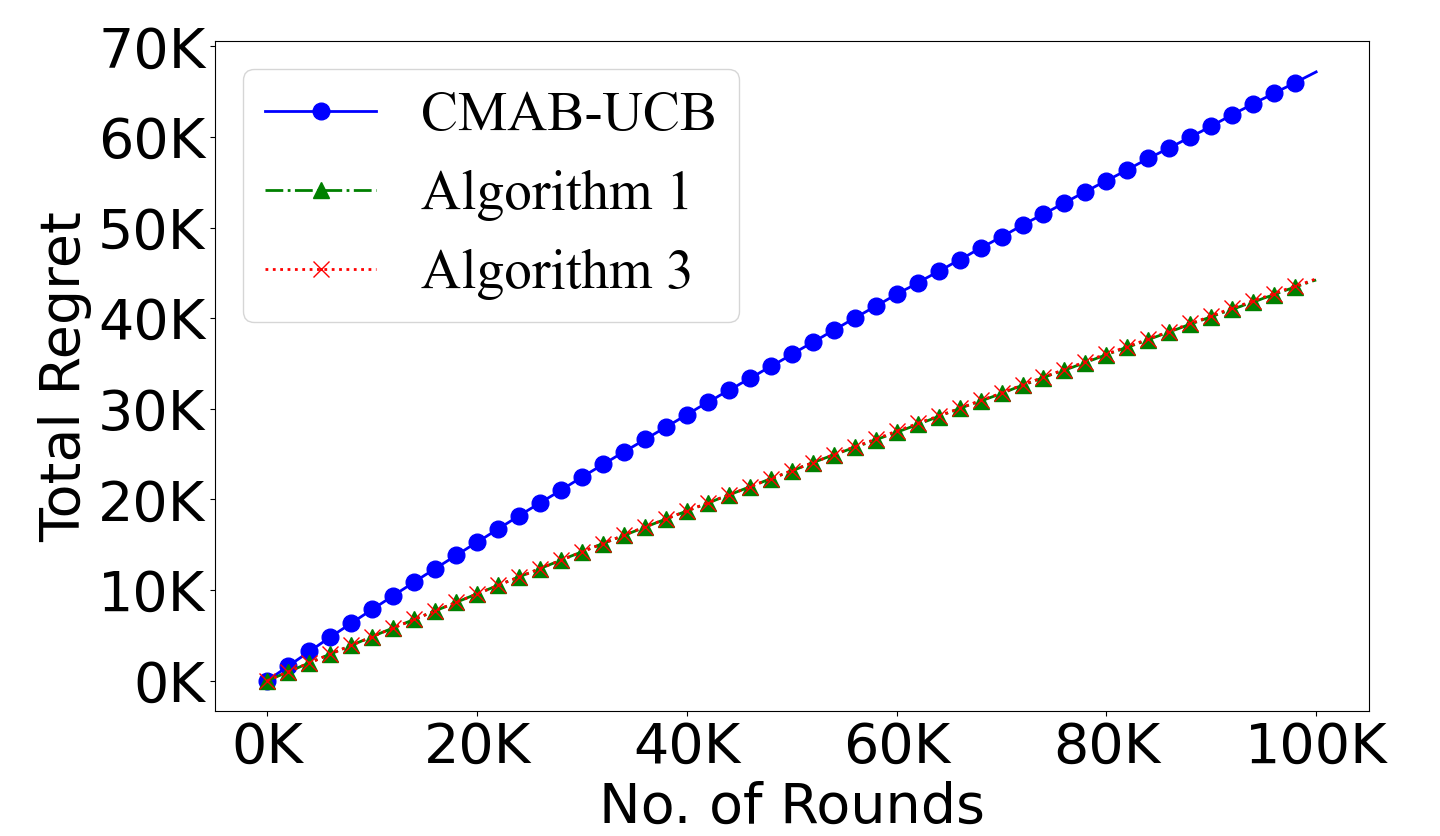} \label{fig:Regret_Unknown_W}}
    \caption{Performance of Algorithms~\ref{alg:ucb-rec} and~\ref{alg:ucb-unk-w}.}
    \label{fig:all_metrics}
\end{figure*}

\section{Numerical Results} \label{Numerical Results}
In this section, we numerically evaluate the UCB-based
algorithm, Algorithm \ref{alg:ucb-rec} and \ref{alg:ucb-unk-w} for various system parameters. We compare the performance of Algorithm \ref{alg:ucb-rec} to that of the standard CMAB-UCB algorithm \cite{DBLP:journals/corr/KvetonWAEE14}, the greedy algorithm, and the $\epsilon$-greedy algorithm. In all the experiments, we assume a total number of contents $N = 50$, cache capacity $C = 20$, the number of users $U = 20$ and $\eta = 4$.

\paragraph{Performance of Algorithm \ref{alg:ucb-rec} compared with CMAB-UCB \cite{DBLP:journals/corr/KvetonWAEE14}, $\epsilon$-greedy and greedy algorithms with uniform distribution over the recommended contents}
We compare the performance of Algorithm \ref{alg:ucb-rec} with CMAB-UCB ,$\epsilon$-Greedy and Greedy algorithm. In the Greedy algorithm, we calculate $p_i(t)$ for all $i \in \mathcal{N}$. Then, pick the top $C$ contents with the highest $p_i(t)$. In the $\epsilon$-Greedy algorithm, a random number is generated every time. If it exceeds the epsilon, the top $C$ contents with the highest $p_i(t)$ are selected. Else $C$ contents selected uniformly from content set
 $\mathcal{N}$. In this experiment, we used $w_u^{rec} = 0.95 $ for all $u \in \mathcal{U}$, $\alpha = 5, \epsilon = 0.4$. We can see from the Fig.\ref{fig:Performance} that our algorithm performs better compared to better to other algorithms.

\paragraph{Performance of Algorithm \ref{alg:ucb-rec} compared with CMAB-UCB \cite{DBLP:journals/corr/KvetonWAEE14} for various $w_u^{rec}$} In this experiment, we compare the performance of  Algorithm \ref{alg:ucb-rec} for various $w_u^{rec}$ values i.e. $w_u^{rec} = 0.99, 0.8, 0.5$ for all $u \in \mathcal{U}$. When $w_u^{rec} \approx 1$, users always tend to accept the recommendation. Since we recommend cached contents, we expect the regret to be low. Hence, we want regret to decrease when $w_u^{rec}$ increases and becomes close to 1. This can be seen from the Fig.\ref{fig:W_rec}

\paragraph{Performance of Algorithm \ref{alg:ucb-rec} compared with CMAB-UCB \cite{DBLP:journals/corr/KvetonWAEE14} for various number of users} In this experiment, we compare the performance of Algorithm~\ref{alg:ucb-rec} for various number of users i.e. $U = 2,10,30$ keeping $w_u^{rec} = 0.95$ for all $u \in \mathcal{U}$. From the Fig.\ref{fig:alpha}, we can see that regret increases with the number of users. This is expected becacuse as the number of users increases, the number of unknowns $p_u^{\text{req}}(\cdot)$ also increases. The increase in the regret due to number of users can also be seen in the Theorem~\ref{Theorem: Regret Results}.  

\paragraph{Performance of Algorithm \ref{alg:ucb-rec} compared with existing algorithms under the Zipf distribution among the recommended contents} In this experiment , we assume Zipf distribution \eqref{Eqn: Zipf Distr} over the recommended contents instead of uniform distribution. This makes the probability of requests to depend on the position of content in the recommended list. Here we also assume $w_u^{\text{rec}}$ is different for $u \in \mathcal{U}$ and are choosen uniformly from $[0.9 , 0.99]$  for each user. Similarly $\beta_u$ is choosen uniformly from $[1,2]$ for  each user. Also $\alpha$ is fixed at $5$. Then we compare the performance of Algorithm \ref{alg:ucb-rec} for both uniform and Zipf distribution over recommended contents with the other algorithms and is shown in the Fig.\ref{fig:Performance_Zipf}. From the Fig.\ref{fig:Performance_Zipf}, we can see that having Zipf distribution over the recommended contents yields the same performance as uniform distribution as we justified in Remark \ref{Remark:Zipf}.

\paragraph{Convergence of $\bar{w}^{\text{rec}}(t)$ to true mean $\bar{w}^{\text{rec}}$} In this experiment,we assume uniform distribution over the recommended contents and also assume that same contents are recommended to all the users. We choose $w_u^{\text{rec}}$ uniformly from $[0.1,0.9]$ for each user $u \in \mathcal{U}$. Also $\alpha$ is fixed at 5. Then we estimate $\bar{w}^{\text{rec}}(t)$ at every time $t$ and convergence is shown in the Fig.\ref{fig:Convergence}. From the Fig. \ref{fig:Convergence}, we can see that $\bar{w}^{\text{rec}}(t)$ converges to true mean $\bar{w}^{\text{rec}}$ as expected.

\paragraph{Regret performance of algorithm \ref{alg:ucb-unk-w} compared with Algorithm \ref{alg:ucb-rec} and CMAB-UCB \cite{DBLP:journals/corr/KvetonWAEE14}} In this experiment we  assume uniform distribution over the recommended contents and also assume that same contents are recommended to all the users. We choose $w_u^{\text{rec}}$ uniformly from $[0.1,0.9]$ for each user $u \in \mathcal{U}$. Also $\alpha$ is fixed at 5. From the Fig. \ref{fig:Regret_Unknown_W}, we can see that Algortihm \ref{alg:ucb-unk-w} performs close to the Algorithm \ref{alg:ucb-rec} with known $\bar{w}^{\text{rec}}$. 


\section{Conclusion and Future Work} \label{Conclusion and Future work} 
In this work, we analysed a content caching problem with  recommendations, modelling it as a CMAB. We proposed a UCB-based algorithm~(Algorithm~\ref{alg:ucb-rec}) and  provided an upper bound on its regret (Theorem~\ref{Theorem: Regret Results}).
We also considered a setup where the users’
recommendation acceptabilities are unknown and proposed another 
 UCB-based algorithm~(Algorithm~\ref{alg:ucb-unk-w}) for this case.
We numerically compared the performance of these algorithms to that of a few existing algorithms. 
We found that caching with recommendation improves the cache hit performance in all the cases.

In our future work,  we would like to incorporate recommendation quality \eqref{eqn: Recommendation Quality} into our problem and would also like to analyze  regret of Algorithm \ref{alg:ucb-unk-w} for the case of unknown  $w_u^{\mathrm{rec}}$'s.
We would also like to consider another recommendation model where each user examines the contents in its recommendation list one by one until it finds an interesting content, and once it finds an interesting content, it does not examine the rest.

\bibliographystyle{plain}
\bibliography{references}

\appendix 

\subsection{Proof of Lemma~\ref{Lemma 1}} \label{Proof:Lemma1}
\begin{proof}
Observe that
    \begin{align*}
        \operatorname{R}(T)  = & T \sum_u \sum_i y_i^* p_u^{\text {req }}(i) - \sum_{t=1}^T \sum_u \sum_{i} y_i^t p_u^{\text {req }}(i) \nonumber \\
        = &  T \sum_{i \in \mathcal{Y}} \sum_u p_u^{\text {req }}(i)-\sum_{t=1}^T \sum_u \sum_{i \in y(t)} p_u^{\text {req }}(i)  \nonumber\\
        \stackrel{(a)}= & T \sum_{i \in \mathcal{Y}} \sum_u \bigg[w_u^{\text{rec}} \frac{x_{ui}^*}{R} + ( 1 - w_u^{\text{rec}}) p_u^{\text{pref}}(i)\bigg] \\
        & - \sum_{t=1}^T \sum_{i \in y(t)}  \sum_u \bigg[w_u^{\text{rec}} \frac{x_{ui}^{(t)}}{R} + ( 1 - w_u^{\text{rec}} ) p_u^{\text{pref}}(i)\bigg] \\
        \stackrel{(b)}=& \sum_{t=1}^T  \bigg[ \sum_{i \in \mathcal{Y}} \sum_u ( 1 - w_u^{\text{rec}} ) p_u^{\text{pref}}(i) \\
        &-\sum_{i \in y(t)} \sum_u  ( 1 - w_u^{\text{rec}} ) p_u^{\text{pref}}(i)\bigg] \\
        &+ \underbrace{T  \sum_u \frac{w_u^{\text{rec}}}{R}  \sum_{i \in \mathcal{Y}}x_{ui}^* - \sum_{t=1}^T \sum_u \frac{w_u^{\text{rec}}}{R}\sum_{i \in y(t)}x_{ui}^{(t)}}_{=0} \\
        =& T \sum_{i \in \mathcal{Y}} \hat{p_i}-\sum_{t=1}^T \sum_{i \in y(t)} \hat{p_i}.
    \end{align*}
In (a), $x_{ui}^* = 1 $ iff $i$ is recommended and belongs to $\mathcal{Y}$, and $x_{ui}^{(t)} = 1$ iff $i$ is recommended and belongs to $y(t)$. In (b), $ \sum_{i \in \mathcal{Y}}x_{ui}^* = \sum_{i \in y(t)}x_{ui}^{(t)} = R $ because $R$ contents are recommended every time.
\end{proof}

\subsection{Proof of Theorem \ref{Theorem: Regret Results}} \label{Proof; Theorem 1}

\begin{proof}
   We consider bijections
  $\pi^{y(t)}: y(t) \mapsto \mathcal{Y}$ satisfying $\pi^{y(t)}(k) = k$ for all $k \in y(t) \cap \mathcal{Y}$,
    which we use for an equivalent expression of the regret. \remove{For a given $y(t)$,
  $\pi^{y(t)}$ can also be seen as a
 bijection from $\{1,2,\cdots,C\}$ to $\{1,2,\cdots,C\}$, where $\pi^{y(t)}(k)$ gives the index of the item in $\mathcal{Y}$ that is paired with the $k$th item in $y(t)$.} We can  express these bijections  in terms of the following indicator functions.
\begin{align}
    \mathds{1}_{e, k}(t) \coloneqq \mathds{1}\left\{k \in y(t),  \pi^{y(t)}(k)=e\right\}. \label{Bijection function}
\end{align}
From Lemma \eqref{Lemma 1} , we have 
\begin{align}
 \operatorname{R}(T)&=T \sum_{i \in Y} \hat{p_i}-\sum_{t=1}^T \sum_{i \in y(t)} \hat{p_i} \nonumber\\
& =\sum_{t=1}^T\left[\sum_{i \in \mathcal{Y}} \hat{p_i}-\sum_{i \in y(t)} \hat{p_i}\right] \nonumber \\
& =\sum_{t=1}^T \sum_{k \in \mathcal{N} \setminus \mathcal{Y}} \sum_{e \in \mathcal{Y}} \Delta_{e, k} \mathds{1}_{e, k}(t)  \label{eqn: Regret}
\end{align}
where the last equality is obtained using the definitions of the indicator functions in~\eqref{Bijection function}. Taking expectation on both sides of \eqref{eqn: Regret}, we get
\begin{align}
    & \mathbb{E}[\operatorname{R}(T)]= \sum_{k \in \mathcal{N} \setminus \mathcal{Y}} \sum_{e \in \mathcal{Y}} \Delta_{e, k} \mathbb{E}\left[\sum_{t=1}^T \mathds{1}_{e, k}(t)\right]  .\label{eqn:Expected Regret} 
\end{align}
Now we bound the expected cumulative regret associated with each content $k \in \mathcal{N} \setminus \mathcal{Y} $. The key idea in this step is to decompose the indicator function into two parts as follows
\begin{align}
\mathds{1}_{e, k}(t)= & \mathds{1}_{e, k}(t) \mathds{1}\left\{n_k(t-1) \leqslant l_{e, k}\right\} \nonumber \\
& + \mathds{1}_{e, k}(t) \mathds{1}\left\{n_k(t-1)>l_{e, k}\right\} \label{Indicator Parts}
\end{align}
where the numbers $l_{e, k}$ are chosen as explained below. We rewrite \eqref{eqn:Expected Regret} using \eqref{Indicator Parts} as follows 
\begin{align} \label{eqn:Split regret}
& \mathbb{E}[\operatorname{R}(T)] = \sum_{k \in \mathcal{N} \setminus \mathcal{Y}} \sum_{e \in \mathcal{Y}} \Delta_{e, k} \mathbb{E}\left[\sum_{t=1}^T \mathds{1}_{e, k}(t) \mathds{1}\left\{n_k(t - 1) \leqslant l_{e, k}\right\}\right]  \nonumber \\
& \quad + \sum_{k \in \mathcal{N} \setminus \mathcal{Y}} \sum_{e \in \mathcal{Y}} \Delta_{e, k} \mathbb{E}\left[\sum_{t=1}^T \mathds{1}_{e, k}(t) \mathds{1}\left\{n_k(t - 1) > l_{e, k}\right\}\right] .  
\end{align}
We bound the two terms in the right hand side of \eqref{eqn:Split regret} using Lemmas \ref{Lemma:Regret 1} and \ref{Lemma :Regret 2}, respectively.

Let us define sets
$\tau_{i,u} \coloneqq \{t\leq T:i\text{ is cached and recommended to user} \quad u \text{ at time} \quad t\}$ 
for all $i$. We can then estimate $\hat{p_i}$'s using $\hat{p_i}(t)$'s and is defined as follows.
\begin{align}
    \hat{p_i}(t) & \coloneqq 
     \frac{\sum_{s=1}^{t} \bigg[Z_i(s) -  \sum_{u} \frac{w_u^\text{rec}}{R} \mathds{1}\{s \in \tau_{i,u} \}\bigg]}{n_i(t)} .\nonumber
\end{align}
Further, recall that  $\mathds{1}_{e, k}(t) = 1$ implies that $k \in y(t)$ and $e = k$ or $e \in \mathcal{Y}\setminus y(t)$.
In either case, $U_k(t) \geqslant U_e(t)$ which in turn implies at least one of the following three events.
\begin{enumerate} \label{Claim 1}
    \item $\hat{p_e}(t) \leqslant \hat{p_e}-D_e(t) \quad \forall e \in \mathcal{Y}$
    \item $\hat{p_k}(t) \geqslant \hat{p_k}+D_k(t)$
    \item $\Delta_{e, k} \leqslant 2 D_k(t). \quad \forall e \in \mathcal{Y}$
\end{enumerate}
\remove{
\textbf{Hoeffding's inequality}:Let $Z_1, \ldots, Z_n$ be independent bounded random variables with $Z_i \in[a, b]$ for all $i$, where $-\infty<a \leq b<\infty$. Then
$$
\mathbb{P}\left(\frac{1}{n} \sum_{i=1}^n\left(Z_i-\mathbb{E}\left[Z_i\right]\right) \geq t\right) \leq \exp \left(-\frac{2 n t^2}{(b-a)^2}\right)
$$
and
$$
\mathbb{P}\left(\frac{1}{n} \sum_{i=1}^n\left(Z_i-\mathbb{E}\left[Z_i\right]\right) \leq-t\right) \leq \exp \left(-\frac{2 n t^2}{(b-a)^2}\right)
$$
for all $t \geq 0$.
}
We can bound the probabilities of Events $1$ and $2$ using Hoeffoding's inequalities as follows.
\begin{align}
    P\bigg[\hat{p_k}(t)-\hat{p_k}  & \geqslant  D_k(t) \bigg] \leqslant \frac{e^{-\alpha \bar{w}^{\text{rec}}}}{t^{4\left( 1 - \bar{w}^{\text{rec}}\right)^{1/\eta}}} \label{Hoefdings Inequality 1}\\
     P\bigg[\hat{p_e}(t)- \hat{p_e}  & \geqslant  D_k(t) \bigg] \leqslant \frac{e^{-\alpha \bar{w}^{\text{rec}}}}{t^{4\left( 1 - \bar{w}^{\text{rec}}\right)^{1/\eta}}}  .\label{Hoefdings Inequality 2}
\end{align}
If $n_k(t-1) >  4 U^2 \left(\frac{2\left( 1 - \bar{w}^{\text{rec}}\right)^{1/\eta} \log T}{\Delta_{e,k}^2}+\frac{\alpha \bar{w}^{\text{rec}}}{2 \Delta_{e, k}^2}\right)$, rearranging the terms gives $$ \Delta_{e, k} > 2 D_k(t),$$
and so, Event~3 cannot happen in this case. 
Hence we choose \begin{equation} \label{eqn: l value}
    l_{e,k}=\left\lceil 4 U^2\left(\frac{2\left( 1 - \bar{w}^{\text{rec}}\right)^{1/\eta}\ \log T}{\Delta_{e,k}^2}+\frac{\alpha \bar{w}^{\text{rec}}}{2 \Delta_{e, k}^2}\right) \right\rceil.
\end{equation}

Now, using (\ref{Hoefdings Inequality 1}) and (\ref{Hoefdings Inequality 2}), and choosing $l_{e,k}$ as stated above, we can prove the following bounds. 
\begin{lemma} \label{Lemma:Regret 1} 
\begin{align}
    \sum_{e \in \mathcal{Y}} &\Delta_{e,k} \mathbb{E} \bigg[ \sum_{t=1}^{T} \mathds{1}_{e, k}(t) \mathds{1}\left\{n_k(t-1) > l_{e,k}\right\} \bigg]  \nonumber \\
    &\leq U^2 \sum_{e \in \mathcal{Y}} \Delta_{e,k} \left( \frac{2e^{-\alpha \bar{w}^{\text{rec}}}}{4\left( 1 - \bar{w}^{\text{rec}}\right)^{1/\eta}-1} \right)
\end{align}
\begin{proof}
See Appendix \ref{Proof:Lemma 2}
\end{proof}

\end{lemma}

\begin{lemma}  \label{Lemma :Regret 2} 
    \begin{align}
    & \sum_{e \in \mathcal{Y}} \Delta_{e,k} \mathbb{E} \bigg[ \sum_{t=1}^{T} \mathds{1}_{e, k}(t) \mathds{1}\left\{n_k(t-1) \leqslant l_{e,k} \right\} \bigg] \nonumber \\
    & {\leqslant} U^2 \left[\frac{16\left( 1 - \bar{w}^{\text{rec}}\right)^{1/\eta} \log T }{ \Delta_{\min,k}} + \frac{4\alpha \bar{w}^{\text{rec}}}{\Delta_{\min,k}} \right]
\end{align}
\end{lemma}
\begin{proof}
See Appendix \ref{Proof:Lemma 3}
\end{proof}
Now using Lemmas \ref{Lemma:Regret 1}, \ref{Lemma :Regret 2} in \eqref{eqn:Split regret} and further summing over $k \in \mathcal{N}\setminus \mathcal{Y}$, we get the regret bound in Theorem~\ref{Theorem: Regret Results}.
\end{proof}

\subsection{Proof of Lemma \ref{Lemma:Regret 1}} \label{Proof:Lemma 2}
\begin{proof} We can bound the second term of \eqref{eqn:Split regret} as 
   \begin{align*}
\lefteqn{
\sum_{t=1}^T \mathds{1}_{e, k}(t) \mathds{1}\left\{n_k(t-1)>l_{e,k}\right\}} \\ & =\sum_{t=l_{e,k}+1}^T \mathds{1}_{e, k}(t) \mathds{1}\left\{n_k(t-1)>l_{e,k}\right\}\\
& \leqslant \sum_{t=l_{e,k}+1}^I \mathds{1}\left\{
U_k(t) \geqslant U_e(t),
n_k(t-1)>l_{e,k} \right\}.
\end{align*}
Taking expectation on both the sides,
\begin{align}
\mathbb{E}&\left[\sum_{t=1}^T \mathds{1}_{e,k}(t) \mathds{1}\left\{n_k(t-1)>l_{e,k}\right\}\right]\nonumber\\
& = \sum_{t=l_{e,k}+1}^T \mathds{P}\left[U_k(t) \geq U_e(t), n_k(t-1)>l_{e,k}\right] \nonumber\\
& \leqslant \sum_{t=l_{e,k}+1}^T \mathds{P}\left[\text{Event 1} \cup \text{Event 2}\right]\nonumber \\
& \leqslant \sum_{t=l_{e,k}+1}^T P\left[\text{Event 1}\right]+P\left[\text{Event 2}\right] \nonumber\\
& \stackrel{(a)}{\leqslant} \sum_{t=l_{e,k}+1}^T 2\left[ \frac{e^{-\alpha \bar{w}^{\text{rec}}}}{t^{4\left( 1 - \bar{w}^{\text{rec}}\right)^{1/\eta}}}\right]\nonumber \\
& \leqslant 2e^{-\alpha 
  \bar{w}^{\text{rec}}} \int_{l_{e,k}}^{\infty} \frac{1}{t^{4\left( 1 - \bar{w}^{\text{rec}}\right)^{1/\eta}}} d t \nonumber\\
& =2 e^{-\alpha 
\left(\bar{w}^{\text{rec}}\right)}\left[\frac{t^{-4\left(1- \bar{w}^{\text{rec}}\right)^{1/\eta}+1}}{-4\left( 1 - \bar{w}^{\text{rec}}\right)^{1/\eta}+1}\right]_{l_{e,k}}^{\infty}  \nonumber\\
& =\frac{2 e^{-\alpha 
\left(\bar{w}^{\text{rec}}\right)}l_{e,k}^{-4\left( 1 - \bar{w}^{\text{rec}}\right)^{1/\eta}+1}   }{4\left( 1 - \bar{w}^{\text{rec}}\right)^{1/\eta}-1} \nonumber\\
&\stackrel{(b)}{\leqslant} \frac{2 e^{-\alpha 
  \bar{w}^{\text{rec}}}}{4\left( 1 - \bar{w}^{\text{rec}}\right)^{1/\eta}-1} \label{Regret 1a}.
\end{align}
where $(a)$ follows from~\eqref{Hoefdings Inequality 1} and \eqref{Hoefdings Inequality 2}, and $(b)$ from the assumptions that $4\left( 1 - \bar{w}^{\text{rec}}\right)^{1/\eta}>1$ and $l_{e,k}>1$. Using~\eqref{Regret 1a}, we can write the following bound.
\begin{align*}
    \sum_{e \in \mathcal{Y}} \Delta_{e,k} \mathbb{E}& \left[ \sum_{t=1}^{T} \mathds{1}_{e, k}(t) \mathds{1}\left\{n_k(t-1) > l_{e, k}\right\} \right]   \\
    &\leq \sum_{e \in \mathcal{Y}} \Delta_{e,k} \left( \frac{2e^{-\alpha \bar{w}^{\text{rec}}}}{4\left( 1 - \bar{w}^{\text{rec}}\right)^{1/\eta}-1} \right). 
\end{align*} 
\end{proof}

\subsection{Proof of Lemma~\ref{Lemma :Regret 2}} 
\label{Proof:Lemma 3}
\begin{proof}
We can write the following bound for $l_{e,k}$ in \eqref{eqn: l value}.
\begin{align}
\sum_{e \in \mathcal{Y}} &\Delta_{e,k} \mathbb{E} \bigg[ \sum_{t=1}^{T} \mathds{1}_{e, k}(t) \mathds{1}\left\{n_k(t-1) \leqslant l_{e, k}\right\} \bigg] \nonumber \\
\leqslant & \max_{y(1),y(2),\dots, y(T)} \bigg[ \sum_{t=1}^{T} \sum_{e \in \mathcal{Y}} \Delta_{e, k} \mathds{1}_{e,k}(t) \mathds{1} \Bigg\{n_k(t-1)    \nonumber\\
& \left. \leqslant 4 U^2 \left(\frac{2\left( 1 - \bar{w}^{\text{rec}}\right)^{1/\eta} \log T}{\Delta_{e,k}^2}+ \frac{\alpha \bar{w}^{\text{rec}}}{2 \Delta_{e, k}^2}\right)\right\} \bigg] \label{eqn: Kveton}
\end{align}

Let us recursively define $e_i$'s for all $i \leq C$ as follows.
\begin{equation*}
e_1 = \arg\max_{e\in \mathcal{Y}} \Delta_{e,k} \\
\end{equation*}
and for all $i \in \{2,\cdots,C\}$,
\begin{equation*}
e_i = \arg\max_{e \in \mathcal{Y}\setminus \{e_1, \cdots,e_{i-1}\}} \Delta_{e,k}.     
\end{equation*}
\remove{
\begin{align*}
    \Delta_{e_1,k} &= \max_{e \in \mathcal{Y}} \Delta_{e,k} \\
    \Delta_{e_2,k} &= \max_{e \in \mathcal{Y} \setminus \{e1\}} \Delta_{e,k} \\
     \Delta_{e_3,k} &= \max_{e \in \mathcal{Y}\setminus \{e_1, e_2\}} \Delta_{e,k} \\ 
     & \vdots \\
     \Delta_{e_{C-1},k} &= \max_{e \in \mathcal{Y}\setminus \{e_1, e_2, e_3,..., e_{C-2}\}} \Delta_{e,k} \\ 
     \Delta_{e_{C},k} &= \max_{e \in \mathcal{Y}\setminus \{e_1, e_2, e_3,..., e_{C-2}, e_{C-1}\}} \Delta_{e,k}
\end{align*}
}
We can  see that $\Delta_{e_1,k} \geq \Delta_{e_2,k} \geq ... \geq \Delta_{e_{C},k}$. Also, $\mathcal{Y} = \{ e_1,e_2,\cdots,e_C\}$.
Now we rewrite the right hand side of~\eqref{eqn: Kveton} as follows.
 \begin{align}
\max_{y(1),y(2),\dots, y(T)}& \bigg[ \sum_{t=1}^{T} \sum_{e \in \mathcal{Y}} \Delta_{e, k} \mathds{1}_{e,k}(t) \mathds{1} \Bigg\{n_k(t-1) \nonumber\\
        &  \left. \leqslant 4 U^2 \left(\frac{2\left( 1 - \bar{w}^{\text{rec}}\right)^{1/\eta} \log T}{\Delta_{e,k}^2}+ \frac{\alpha \bar{w}^{\text{rec}}}{2 \Delta_{e, k}^2}\right)\right\} \bigg] \nonumber\\
        \stackrel{(a)}=&  \max_{n_{e_1}, n_{e_2} \dots,n_{e_C}} \sum_{i=1}^C n_{e_i,k} \Delta_{e_i,k} \label{eqn: New Maximum} 
        \end{align}
where in  $(a)$ \begin{equation}
    n_{e_i,k} \coloneqq \sum_{t=1}^{T} \mathds{1} \left\{ k \in y(t) ,\pi^{y(t)}(k)=e_i,  \nonumber \\
    n_k(t-1) \leqslant \frac{M}{\Delta_{e_i, k}^2} \right\}
\end{equation}  and 
$M = 8 U^2\left( 1 - \bar{w}^{\text{rec}}\right)^{1/\eta} \log T + 2 U^2\alpha \bar{w}^{\text{rec}}$. From the above definition of $n_{e_i,k}$  we observe that
it is the number of times content $k$ is selected,
$\pi^{y(t)}(k) = e_i$,
 and $n_k(t - 1) \leqslant M/\Delta_{e_i, k}^2$. Moreover, $\sum_{i = 1}^{j} n_{e_i, k}$ is the number of times content $k$ selected and
 \[
 n_k(t - 1) \leqslant \min_{i \in \{1, 2, \ldots,j\}} \frac{M}{\Delta_{e_i, k}^2}. 
 \]
Therefore 
\[
     \sum_{i = 1}^{j} n_{e_i, k}\leq \frac{M}{\Delta_{e_j, k}^2} 
\]
for all $j \in \{1,2,\cdots, C\}$. Since $\Delta_{e_1,k} \geq \Delta_{e_2,k} \geq ,\cdots, \geq \Delta_{e_{C},k}$, we can set $n_{e_i, k}$ for all $i \in \{1,2,\cdots,C\}$ as follows to maximize~\eqref{eqn: New Maximum}.
 \[
n_{e_1, k} = \frac{M}{\Delta_{e_1, k}^2}\]
and for $i \in \{2,\cdots,C\}$, 
 \[
n_{e_i, k} = \frac{M}{\Delta_{e_{i-1}, k}^2} - \frac{M}{\Delta_{e_i, k}^2}.
\]   
\remove{
\begin{align*}
n_{e_2, k} &= \frac{M}{\Delta_{e_1, k}^2} - \frac{M}{\Delta_{e_1, k}^2} \\
    &\vdots \\
    &\vdots \\
    n_{e_{C-1}, k} &= \frac{M}{\Delta_{e_{C-1}, k}^2} - \frac{M}{\Delta_{e_{C}, k}^2} \\
    n_{e_{C}, k} &= \frac{M}{\Delta_{e_{C-1}, k}^2} - \frac{M}{\Delta_{e_{C}, k}^2}
\end{align*}
}
Hence \begin{align*}
    &\max_{n_{e_1}, n_{e_2} \dots,n_{e_C}} \sum_{i=1}^C n_{e_i,k} \Delta_{e_i,k} \nonumber \\   
    & = M\left [ \Delta_{e_1,k } \frac{1}{\Delta_{e_1,k }^2}+  \sum_{i=2}^{C}  \Delta_{e_i,k } \left( 
 \frac{1}{\Delta_{e_i,k }^2} - \frac{1}{\Delta_{e_{i-1},k }^2}\right)            \right].
\end{align*}
Therefore~\eqref{eqn: Kveton} is bounded as follows 
\begin{align*}
    & \sum_{e \in \mathcal{Y}} \Delta_{e,k} \mathbb{E} \bigg[ \sum_{t=1}^{T} \mathds{1}_{e, k}(t) \mathds{1}\left\{n_k(t-1) \leqslant l_{e, k}\right\} \bigg] \nonumber \\
    &\leq  M\left [ \Delta_{e_1,k } \frac{1}{\Delta_{e_1,k }^2}+  \sum_{i=2}^{C}  \Delta_{e_i,k } \left( 
 \frac{1}{\Delta_{e_i,k }^2} - \frac{1}{\Delta_{e_{i-1},k }^2}\right)            \right].
\end{align*}
Finally, we use following  lemma to get a tighter bound.
\begin{lemma}\label{Lemma 2}~\cite[Lemma~3]{DBLP:journals/corr/KvetonWAEE14} Let $\Delta_1 \geq \Delta_2 \geq ,...\geq \Delta_K$ be a sequence of $K$ positive numbers. Then
\begin{equation}
    \left[ \frac{1}{\Delta_{1}}+  \sum_{k=2}^{K}  \Delta_{k} \left( 
 \frac{1}{\Delta_{k}^2} - \frac{1}{\Delta_{k-1}^2}\right)     \right] \leq \frac{2}{\Delta_K}. \nonumber
\end{equation}
\end{lemma}

Using Lemma~\ref{Lemma 2} we get
\begin{align*}
    \sum_{e \in \mathcal{Y}} \Delta_{e,k}&\mathbb{E} \bigg[ \sum_{t=1}^{T} \mathds{1}_{e, k}(t) \mathds{1}\left\{n_k(t-1) \leqslant l_{e, k}\right\} \bigg] \nonumber \\
    &{\leqslant}  \frac{2M}{\Delta_{e_C,k}} \\
    &= \frac{16 U^2 \left( 1 - \bar{w}^{\text{rec}}\right)^{1/\eta} \log T }{ \Delta_{\min,k}} + \frac{4 U^2 \alpha \bar{w}^{\text{rec}}}{\Delta_{\min,k}}.
\end{align*}
\end{proof}

\end{document}